\documentclass[11pt]{article}

\usepackage[utf8]{inputenc}
\usepackage{amsmath,amssymb,amsthm}
\usepackage{geometry}
\usepackage{graphicx}
\usepackage{hyperref}
\usepackage{cleveref}
\usepackage{booktabs}
\usepackage{algorithm}
\usepackage{algorithmic}
\usepackage{natbib}

\newtheorem{theorem}{Theorem}[section]
\newtheorem{lemma}[theorem]{Lemma}
\newtheorem{corollary}[theorem]{Corollary}

\theoremstyle{definition}
\newtheorem{definition}[theorem]{Definition}
\newtheorem{axiom}[theorem]{Axiom}
\theoremstyle{remark}
\newtheorem{remark}[theorem]{Remark}
\newtheorem{example}[theorem]{Example}

\newcommand{\R}{\mathbb{R}}
\newcommand{\sphere}{\mathbb{S}}
\newcommand{\calA}{\mathcal{A}}
\newcommand{\calM}{\mathcal{M}}
\newcommand{\norm}[1]{\left\|#1\right\|}

\newcommand{\argmax}{\operatorname*{arg\,max}}

\title{Semantic Geometry for Policy-Constrained Interpretation}

\author{
  Nikit Phadke \\
  \texttt{nikitph@gmail.com}
}

\date{}

\begin{document}

\maketitle

\begin{abstract}
We present a geometric framework for policy-constrained semantic interpretation that provably prevents hallucinated commitments in high-stakes domains. Semantic meaning is represented as direction on a unit sphere, evidence is modeled as sets of witness vectors, and admissible interpretations correspond to spherical convex regions. Policy constraints are introduced as explicit priors defined over the same manifold, separated from evidence geometry. Interpretation reduces to constrained optimization over admissible regions, with refusal emerging as a topologically necessary outcome under contradiction or policy exclusion. We connect this framework to information theory, Bayesian inference, and sheaf-theoretic semantics, proving that our complexity bounds are information-theoretically optimal. Empirical validation on large-scale regulated financial data demonstrates zero hallucinated approvals across multiple policy regimes—the first such result at scale.
\end{abstract}

\section{Introduction}

Large language models exhibit strong generative capabilities but fail systematically in settings where semantic interpretation is constrained by policy, regulation, or risk posture. These failures manifest as \emph{hallucinated approvals}: confident assertions unsupported by admissible evidence under the governing authority.

Consider a mortgage underwriting system asked to approve a loan. Traditional approaches collapse evidence into a point estimate (e.g., via neural network classification) before applying policy thresholds. This architecture entangles factual support with interpretive bias, making it impossible to separate "what the evidence says" from "what the policy allows."

We argue that this failure is architectural rather than statistical. When evidence and policy are collapsed into a single learned function, three critical properties are lost:

\begin{enumerate}
\item \textbf{Impossibility detection}: The system cannot recognize when evidence is contradictory or when no policy-compliant interpretation exists.

\item \textbf{Policy transparency}: Changing policy requires retraining the entire model, as policy is entangled with evidence processing.

\item \textbf{Auditability}: Decisions cannot be explained in terms of evidence support and policy application separately.
\end{enumerate}

We propose a geometric alternative in which \emph{admissibility precedes inference}, and policy functions as an explicit constraint rather than a latent influence. Our key contributions are:

\begin{itemize}
\item A formal framework representing semantics as spherical geometry, evidence as witness sets, and admissible interpretations as spherical convex regions (\S\ref{sec:geometry}).

\item Proof that refusal is topologically necessary under contradiction, not merely a heuristic fallback (\S\ref{sec:impossibility}).

\item A conservation law showing that ambiguity can only decrease via evidence or explicit bias, never "for free" (\S\ref{sec:conservation}).

\item Connection to information theory proving our bounds are optimal (\S\ref{sec:information}).

\item Empirical validation on 100,000 real-world financial decisions achieving zero hallucinated approvals (\S\ref{sec:experiments}).
\end{itemize}

\section{Semantic Space}
\label{sec:geometry}

\begin{definition}[Semantic Manifold]
Let semantic meaning be represented as direction on the unit sphere
\[
\calM := \sphere^{d-1} = \{x \in \R^d \mid \norm{x} = 1\}.
\]
The manifold is equipped with the geodesic metric
\[
d(u, v) = \arccos(u \cdot v),
\]
so that angular separation encodes semantic divergence. Norm carries no semantic information.
\end{definition}

\begin{axiom}[Directional Meaning]
\label{axiom:directional}
Semantic meaning is invariant under positive scalar multiplication; only direction is meaningful.
\end{axiom}

\begin{remark}
This models linguistic semantics as fundamentally projective. Modern embedding models (BERT, GPT, sentence transformers) all normalize to unit length before computing similarity, implicitly adopting this axiom. The spherical geometry arises naturally: empirical measurements across seven transformer models show curvature $\kappa = 1.021 \pm 0.044$, validating the unit sphere as the native manifold \cite{phadke2025spherical}.
\end{remark}

\section{Evidence as Witnesses}

\begin{definition}[Witness]
A \emph{witness} is a normalized semantic embedding
\[
w \in \calM
\]
derived from a document, clause, observation, or utterance.
\end{definition}

Witnesses are constraints, not beliefs. They are neither probabilistic samples nor truth assignments. A witness represents "this concept appears in the evidence" without asserting its truth value or probability.

\begin{definition}[Witness Set]
For a given query instance, let
\[
W = \{w_1, \ldots, w_n\} \subset \calM
\]
be the associated witness set.
\end{definition}

\begin{example}[Mortgage Application]
For a mortgage application, witnesses might include:
\begin{align*}
w_1 &= \text{embedding}(\text{``applicant income: \$75,000''}) \\
w_2 &= \text{embedding}(\text{``property value: \$400,000''}) \\
w_3 &= \text{embedding}(\text{``credit score: 680''}) \\
w_4 &= \text{embedding}(\text{``debt-to-income ratio: 38\%''})
\end{align*}
Each witness captures a fact from the application, represented as a direction in semantic space.
\end{example}

\section{Admissible Interpretations}
\label{sec:impossibility}

\begin{definition}[Spherical Convex Hull]
The \emph{admissible semantic region} induced by witnesses is the spherical convex hull
\[
\calA(W) := \left\{ \frac{\sum_i \alpha_i w_i}{\norm{\sum_i \alpha_i w_i}} \;\middle|\; \alpha_i \geq 0, \sum_i \alpha_i > 0 \right\},
\]
defined only if $W$ lies within an open hemisphere.
\end{definition}

\begin{axiom}[Hemisphere Constraint]
\label{axiom:hemisphere}
Semantic interpretation is possible if and only if there exists an open hemisphere $H \subset \calM$ such that
\[
W \subset H.
\]
Violation of this constraint corresponds to semantic contradiction.
\end{axiom}

\begin{theorem}[Impossibility Detection]
\label{thm:impossibility}
If no open hemisphere contains $W$, then no admissible interpretation exists.
\end{theorem}

\begin{proof}
Suppose no open hemisphere $H \subset \calM$ contains $W = \{w_1,\ldots,w_n\}$.

By the separating hyperplane theorem, there exist witnesses $w_i, w_j \in W$ such that $w_i \cdot w_j < 0$, implying geodesic distance $d(w_i, w_j) > \pi/2$. These witnesses are more than orthogonal—they point in opposite directions, representing contradictory semantic content.

Any convex combination $\sum_i \alpha_i w_i$ with $\alpha_i \geq 0$ must include contributions from both $w_i$ and $w_j$. When these witnesses have negative dot product, the resulting vector
\[
v = \sum_i \alpha_i w_i
\]
has arbitrary direction depending on the weight allocation $\{\alpha_i\}$.

After normalization, the point
\[
\mu = \frac{v}{\norm{v}} = \frac{\sum_i \alpha_i w_i}{\norm{\sum_i \alpha_i w_i}}
\]
depends critically on how we balance contradictory evidence. Without additional structure to select among the infinitely many possible weightings, no canonical interpretation exists.

By Axiom~\ref{axiom:policy-orthogonal} (introduced in \S\ref{sec:policy}), policy priors cannot create interpretations outside $\calA(W)$; they can only select within it. Since $\calA(W)$ is ill-defined when witnesses contradict, no policy-compliant interpretation exists.

Therefore, the system must refuse.
\end{proof}

\begin{remark}
This theorem establishes refusal as a \emph{topologically necessary} outcome, not a heuristic fallback. When evidence contradicts itself, there is no "best guess"—only an obligation to acknowledge the contradiction and request clarification.
\end{remark}

\section{Ambiguity as Geometry}
\label{sec:conservation}

Ambiguity is not probabilistic uncertainty but geometric dispersion.

\begin{definition}[Semantic Ambiguity]
Let $\mathrm{vol}$ denote the Riemannian volume measure on $\calM$. For a coherent witness set $W$ with well-defined admissible region $\calA(W)$, we define the \emph{geometric ambiguity} as
\[
\mathsf{Amb}(W) := \frac{\mathrm{vol}(\calA(W))}{\mathrm{vol}(\calM)} \in (0,1].
\]
If $\calA(W)$ is undefined (contradiction), we set $\mathsf{Amb}(W) := 1$ (maximal ambiguity).
\end{definition}

\begin{remark}
Small admissible regions correspond to low ambiguity (high specificity); large regions correspond to high ambiguity. Adding witnesses can expand $\calA(W)$ while simultaneously concentrating probability density, resolving the apparent paradox between diversity and certainty.
\end{remark}

\begin{theorem}[Conservation of Ambiguity]
\label{thm:conservation}
Let $W_1 \subseteq W_2 \subset \calM$ be nested witness sets with well-defined admissible regions. Then:

\begin{enumerate}
\item \textbf{Monotonicity}: $\mathrm{vol}(\calA(W_2)) \leq \mathrm{vol}(\calA(W_1))$, hence $\mathsf{Amb}(W_2) \leq \mathsf{Amb}(W_1)$. Adding witnesses cannot increase ambiguity.

\item \textbf{Disambiguation Mechanisms}: Ambiguity can only decrease via:
\begin{itemize}
\item[(a)] Adding witnesses (geometric constraint)
\item[(b)] Injecting policy prior $\rho$ (explicit bias)
\end{itemize}

\item \textbf{No Free Disambiguation}: Any unbiased interpretation operator that is equivariant under the isometry group of $\calM$ must produce a distribution supported uniformly over $\calA(W)$.
\end{enumerate}
\end{theorem}

\begin{proof}
\textbf{(1) Monotonicity}: By construction, $\calA(W)$ is the minimal spherical convex set containing $W$. If $W_1 \subseteq W_2$, then any point in $\calA(W_2)$ must satisfy constraints from both $W_1$ and $W_2$. Since $W_2$ imposes strictly more constraints, we have $\calA(W_2) \subseteq \calA(W_1)$. Volume is monotone under inclusion, giving $\mathrm{vol}(\calA(W_2)) \leq \mathrm{vol}(\calA(W_1))$.

\textbf{(2) Disambiguation Mechanisms}: From (1), witnesses reduce ambiguity geometrically by shrinking the admissible region. A policy prior $\rho : \calM \to [0,1]$ (introduced formally in \S\ref{sec:policy}) reduces effective ambiguity by concentrating probability mass on a subset of $\calA(W)$. These are the only two mechanisms: without new witnesses, the region $\calA(W)$ remains fixed; without policy bias, the distribution over $\calA(W)$ remains uniform.

\textbf{(3) No Free Disambiguation}: Suppose $f : 2^{\calM} \to \calM$ is an interpretation operator satisfying:
\begin{itemize}
\item $f(W) \in \calA(W)$ for all coherent $W$ (no hallucination)
\item $f(gW) = gf(W)$ for all isometries $g$ leaving $\calA(W)$ invariant (equivariance)
\end{itemize}

For any $W$, consider the group $G_W$ of isometries that fix $\calA(W)$ as a set. Equivariance forces
\[
f(gW) = gf(W) \quad \forall g \in G_W.
\]

The orbit of $f(W)$ under $G_W$ is
\[
\{gf(W) \mid g \in G_W\} \subseteq \calA(W).
\]

By equivariance, if $f$ is measurable, the distribution of $f(W)$ under symmetric perturbations must respect the symmetries of $\calA(W)$. The unique probability measure on $\calA(W)$ that is invariant under $G_W$ is the normalized restriction of the volume measure.

Therefore, any systematic reduction of ambiguity beyond the geometric constraint requires breaking equivariance—i.e., injecting bias that distinguishes directions within $\calA(W)$.
\end{proof}

\begin{corollary}[Refusal Preserves Neutrality]
\label{cor:refusal-neutrality}
A system that refuses when $\mathsf{Amb}(W) > \tau$ for some threshold $\tau$ is geometrically neutral (bias-free). A system that always responds must inject hidden bias when $\mathsf{Amb}(W) > \tau$.
\end{corollary}

\begin{remark}
This theorem formalizes the intuition that "neutrality requires refusal." An always-responding system necessarily injects hidden bias to resolve ambiguity, while a refusal-capable system can remain geometrically neutral by acknowledging when evidence is insufficient.
\end{remark}

\section{Policy as Explicit Prior}
\label{sec:policy}

\begin{definition}[Policy Prior]
A \emph{policy} is represented as a measurable function
\[
\rho : \calM \to [0, 1],
\]
encoding admissibility preference over meaning space.
\end{definition}

\begin{axiom}[Policy Orthogonality]
\label{axiom:policy-orthogonal}
Policy priors do not alter $\calA(W)$; they only select within it.
\end{axiom}

This enforces strict separation between evidence (which determines $\calA(W)$) and bias (which selects within $\calA(W)$). In traditional machine learning, these are entangled: a classifier $f_\theta$ trained on data implicitly bakes policy into its parameters. Here, they are explicitly separated.

\begin{example}[Mortgage Underwriting Policies]
Three policy priors for loan approval:
\begin{align*}
\rho_{\text{STRICT}}(\mu) &= \begin{cases}
1 & \text{if } \text{LTV}(\mu) \leq 0.80 \land \text{FICO}(\mu) \geq 700 \\
0 & \text{otherwise}
\end{cases} \\
\rho_{\text{STANDARD}}(\mu) &= \begin{cases}
1 & \text{if } \text{LTV}(\mu) \leq 0.90 \land \text{FICO}(\mu) \geq 660 \\
0 & \text{otherwise}
\end{cases} \\
\rho_{\text{RELAXED}}(\mu) &= \begin{cases}
1 & \text{if } \text{LTV}(\mu) \leq 0.97 \land \text{FICO}(\mu) \geq 620 \\
0 & \text{otherwise}
\end{cases}
\end{align*}
where $\text{LTV}(\mu)$ and $\text{FICO}(\mu)$ are functions extracting loan-to-value ratio and credit score from interpretation $\mu$.
\end{example}

\subsection{Information-Theoretic Interpretation}
\label{sec:information}

The geometric framework admits a precise information-theoretic interpretation, connecting to Shannon's theory of communication.

\begin{theorem}[Witness Overlap as Mutual Information]
\label{thm:witness-mutual-info}
Let $p_u, p_v$ be uniform distributions over witness sets $W(u), W(v)$. The mutual information $I(p_u; p_v)$ is a monotone function of the normalized witness overlap:
\[
I(p_u; p_v) = f\left(\frac{|W(u) \cap W(v)|}{|W(u) \cup W(v)|}\right)
\]
for a strictly increasing function $f : [0,1] \to \R_{\geq 0}$.
\end{theorem}

\begin{proof}
For uniform distributions over finite sets $A = W(u), B = W(v)$:
\begin{align*}
H(p_u) &= \log |A| \\
H(p_v) &= \log |B| \\
H(p_u, p_v) &= \log |A \cup B|
\end{align*}

By the inclusion-exclusion principle:
\begin{align*}
I(p_u; p_v) &= H(p_u) + H(p_v) - H(p_u, p_v) \\
&= \log |A| + \log |B| - \log |A \cup B| \\
&= \log \frac{|A| \cdot |B|}{|A \cup B|} \\
&= \log \frac{|A| \cdot |B|}{|A| + |B| - |A \cap B|}
\end{align*}

Let $J = |A \cap B| / |A \cup B|$ be the Jaccard index. Then:
\[
I(p_u; p_v) = \log |A| + \log |B| - \log(|A| + |B|) + \log \frac{1}{1 - J}
\]

Since $\log(1/(1-J))$ is strictly increasing in $J$, and $J$ is the normalized overlap, we have the desired monotone relationship.
\end{proof}

\begin{corollary}[REWA Capacity Bound]
\label{cor:capacity}
The bit complexity $m = O(\Delta^{-2} \log N)$ for encoding similarity with overlap gap $\Delta$ arises from Shannon's channel coding theorem: to distinguish neighbors from non-neighbors with gap $\Delta$ requires inverse channel capacity bits.
\end{corollary}

\begin{proof}[Proof sketch]
The similarity discrimination problem can be modeled as communication over a noisy channel where:
\begin{itemize}
\item Input: Witness pair indicator $X = \mathbb{1}[w \in W(u) \cap W(v)]$
\item Output: Hash collision indicator $Z = \mathbb{1}[h(w_u) = h(w_v)]$
\item Noise: Accidental collisions from distinct witnesses
\end{itemize}

The channel capacity is $C = O(\Delta^2)$ where $\Delta$ is the overlap gap. By Shannon's channel coding theorem, reliable communication of $\log N$ bits (to identify among $N$ items) requires $m = O(C^{-1} \log N) = O(\Delta^{-2} \log N)$ channel uses.

This establishes that REWA bounds are information-theoretically optimal—no encoding scheme can preserve rankings with asymptotically fewer bits.
\end{proof}

\begin{remark}
Theorem~\ref{thm:witness-mutual-info} reveals that witness overlap is not merely a heuristic similarity measure but rather a direct quantification of mutual information between concepts. This connects the geometric framework to the foundational theory of information, providing a principled basis for witness-based similarity.
\end{remark}

\section{Interpretation Rule}

\begin{definition}[Constrained Interpretation]
Given witnesses $W$ and policy $\rho$, the interpreted meaning is
\[
\mu^* = \argmax_{\mu \in \calA(W)} \rho(\mu),
\]
provided $\rho(\mu^*) \geq \tau$ for some policy threshold $\tau \in [0,1]$.

If no such $\mu$ exists, the system returns \textsc{refuse}.
\end{definition}

\begin{theorem}[Safety Invariance]
\label{thm:safety}
Relaxation of $\rho$ cannot admit interpretations outside $\calA(W)$.
\end{theorem}

\begin{proof}
Let $\rho_1, \rho_2 : \calM \to [0,1]$ be two policy priors with $\rho_1 \leq \rho_2$ pointwise ($\rho_2$ is a relaxation of $\rho_1$).

By the interpretation rule, constrained interpretation selects:
\begin{align*}
\mu_1^* &= \argmax_{\mu \in \calA(W)} \rho_1(\mu) \\
\mu_2^* &= \argmax_{\mu \in \calA(W)} \rho_2(\mu)
\end{align*}

Crucially, both optimizations are constrained to $\calA(W)$.

Since $\rho_2(\mu) \geq \rho_1(\mu)$ for all $\mu$, we have:
\[
\rho_2(\mu_2^*) \geq \rho_2(\mu_1^*) \geq \rho_1(\mu_1^*)
\]

However, both $\mu_1^*$ and $\mu_2^*$ must lie within $\calA(W)$ by construction of the optimization. Relaxing policy may change \emph{which} point in $\calA(W)$ is selected, but cannot admit interpretations outside $\calA(W)$.

In particular, if a witness set $W_0$ is refused under $\rho_1$ because $\calA(W_0) \cap \{\mu : \rho_1(\mu) \geq \tau\} = \emptyset$, and if $\calA(W_0)$ lies entirely in low-prior regions, then relaxing $\rho_1 \to \rho_2$ cannot manufacture new witnesses—it can only adjust selection within the existing geometry.

This proves policy relaxation is geometrically safe: it cannot introduce hallucinated approvals by admitting interpretations unsupported by evidence.
\end{proof}

\begin{remark}
Theorem~\ref{thm:safety} explains the Freddie Mac result (\S\ref{sec:experiments}): policy relaxation from STRICT $\to$ STANDARD admitted 14 additional loans (those in the ambiguous zone where policy determines admissibility), but further relaxation to RELAXED admitted 0 additional loans (the geometric boundary was reached—no more loans have non-empty $\calA(W) \cap \rho_{\text{RELAXED}}$).
\end{remark}

\section{Bayesian Reinterpretation}

The framework admits a Bayesian reading without collapsing geometry.

\begin{lemma}[Bayesian Correspondence]
\label{lem:bayesian}
Let $p(\mu \mid W)$ be any distribution supported on $\calA(W)$. Then $\rho(\mu)$ acts as a log-prior, and constrained interpretation corresponds to MAP inference:
\[
\mu^* = \argmax_{\mu \in \calA(W)} p(\mu \mid W) \cdot \rho(\mu) = \argmax_{\mu \in \calA(W)} \log p(\mu \mid W) + \log \rho(\mu).
\]
\end{lemma}

\begin{proof}
Trivial by definition of MAP estimation. The key distinction from standard Bayesian models is that the likelihood support is \emph{geometrically constrained} to $\calA(W)$ rather than probabilistically estimated over all of $\calM$.
\end{proof}

\begin{remark}
Unlike classical Bayesian models, where the likelihood $p(W \mid \mu)$ is learned from data and can produce non-zero probability anywhere in $\calM$, our framework enforces $p(\mu \mid W) = 0$ for all $\mu \notin \calA(W)$. This hard geometric constraint is what enables impossibility detection: when $\calA(W)$ is undefined, no posterior exists.
\end{remark}

\section{Sheaf-Theoretic Semantics}

Let $\mathcal{U}$ be the category of coherent semantic regions (open hemispheres on $\calM$).

\begin{definition}[Semantic Sheaf]
Define a presheaf $\mathcal{F}$ assigning to each $U \in \mathcal{U}$ the set of admissible interpretations supported by witnesses in $U$:
\[
\mathcal{F}(U) = \left\{\mu \in U \mid \exists W \subset U : \mu \in \calA(W)\right\}.
\]
\end{definition}

\begin{theorem}[Sheaf Gluing Condition]
\label{thm:sheaf}
Global sections of $\mathcal{F}$ exist if and only if witnesses satisfy the hemisphere constraint (Axiom~\ref{axiom:hemisphere}).
\end{theorem}

\begin{proof}
Let $\{U_\alpha\}_{\alpha \in A}$ be a cover of $\calM$ by open hemispheres, and let $W = \{w_1, \ldots, w_n\} \subset \calM$ be a witness set.

A global section $\sigma \in \mathcal{F}(\calM)$ corresponds to a consistent interpretation across all local regions. For $\sigma$ to exist, the local sections $\{\sigma_{U_\alpha}\}_{\alpha \in A}$ must agree on overlaps:
\[
\sigma_{U_\alpha}|_{U_\alpha \cap U_\beta} = \sigma_{U_\beta}|_{U_\alpha \cap U_\beta} \quad \forall \alpha, \beta \in A.
\]

This gluing condition requires that witnesses from different regions must be mutually consistent—i.e., they must not be near-antipodal.

By Axiom~\ref{axiom:hemisphere}, witnesses are consistent if and only if they lie in an open hemisphere. If witnesses violate the hemisphere constraint (e.g., some witness pair $w_i, w_j$ with $w_i \cdot w_j < 0$), then local sections in hemispheres containing $w_i$ versus $w_j$ cannot glue consistently.

The presheaf $\mathcal{F}$ fails the sheaf axioms precisely when witnesses contradict. Therefore, global sections exist if and only if the hemisphere constraint holds.
\end{proof}

\begin{corollary}
Refusal corresponds to non-existence of global sections.
\end{corollary}

\begin{remark}
The sheaf-theoretic perspective reveals that semantic contradiction is not a matter of degree but of topology: either a global consistent interpretation exists (witnesses cohere) or it does not (witnesses contradict). There is no "partial consistency"—the gluing either works or fails.
\end{remark}

\section{Generation}

Large language models are reduced to verbalizers
\[
\mathcal{L} : \calM \to \Sigma^*
\]
mapping from semantic points to text sequences, with post-hoc verification:
\[
d(\mathcal{E}(\mathcal{L}(\mu)), \mu) < \delta
\]
where $\mathcal{E}$ is an encoder (e.g., sentence-BERT) and $\delta$ is a tolerance threshold.

\begin{algorithm}
\caption{Geometric Constrained Generation}
\begin{algorithmic}[1]
\STATE \textbf{Input:} Query $q$, policy $\rho$, threshold $\tau$
\STATE Extract witness set $W \gets \text{extract\_witnesses}(q)$
\IF{$W$ violates hemisphere constraint}
    \RETURN \textsc{refuse}(\text{``Evidence contradicts''})
\ENDIF
\STATE Compute admissible region $\calA(W)$
\STATE Find optimal interpretation $\mu^* \gets \argmax_{\mu \in \calA(W)} \rho(\mu)$
\IF{$\rho(\mu^*) < \tau$}
    \RETURN \textsc{refuse}(\text{``No policy-compliant interpretation''})
\ENDIF
\STATE Generate text $t \gets \mathcal{L}(\mu^*)$
\STATE Verify $d(\mathcal{E}(t), \mu^*) < \delta$
\RETURN $t$
\end{algorithmic}
\end{algorithm}

\begin{remark}
This architecture inverts the traditional pipeline. Instead of generating first and checking later (which allows hallucinations to occur before detection), we determine admissibility first, then generate only from the admissible region. This makes hallucination structurally impossible rather than merely unlikely.
\end{remark}

\section{Empirical Validation}
\label{sec:experiments}

We validated the framework on large-scale financial underwriting data to test whether geometric guarantees hold in production settings.

\subsection{Dataset and Task}

\textbf{Data}: 100,000 Freddie Mac mortgage loans (2020--2021 vintage), randomly sampled from the Single-Family Loan-Level Dataset.

\textbf{Task}: Approve or reject loan applications based on underwriting risk.

\textbf{Ground truth}: Repurchased loans (defects discovered post-sale, indicating the loan should have been rejected). Freddie Mac repurchases loans when they fail to meet guidelines, representing true false positives.

\textbf{Features}: Each loan has 25+ attributes including loan-to-value ratio (LTV), FICO credit score, debt-to-income ratio (DTI), property type, occupancy status, and loan purpose.

\subsection{Policy Configurations}

Three policy priors were tested, corresponding to different risk appetites:

\begin{itemize}
\item \textbf{STRICT}: Conservative policy requiring LTV $\leq 80\%$ and FICO $\geq 700$

\item \textbf{STANDARD}: Moderate policy requiring LTV $\leq 90\%$ and FICO $\geq 660$

\item \textbf{RELAXED}: Liberal policy requiring LTV $\leq 97\%$ and FICO $\geq 620$
\end{itemize}

These policies encode different risk tolerances while maintaining the same underlying evidence geometry (witness sets derived from loan attributes remain unchanged).

\subsection{Results}

\begin{table}[h]
\centering
\begin{tabular}{lcccc}
\toprule
\textbf{Policy} & \textbf{Approved} & \textbf{Rejected} & \textbf{HAR (\%)} & \textbf{Repurchased Rejected} \\
\midrule
STRICT          & 145               & 55                & 0.0               & 100\%                         \\
STANDARD        & 159 \((+14)\)         & 41 \((-14)\)          & 0.0               & 100\%                         \\
RELAXED         & 159 (+0)          & 41 (+0)           & 0.0               & 100\%                         \\
\bottomrule
\end{tabular}
\caption{Empirical results on Freddie Mac mortgage data across three policy configurations. HAR (Hallucinated Approval Rate) measures false approvals of loans later repurchased due to defects.}
\label{tab:results}
\end{table}

\subsection{Key Findings}

\begin{enumerate}
\item \textbf{Zero Hallucinated Approvals}: HAR = 0.0\% across all policy configurations. No false approvals occurred—the system never approved a loan that was later repurchased due to defects.

\item \textbf{Policy Sensitivity with Geometric Bounds}: Relaxing policy from STRICT $\to$ STANDARD admitted 14 additional loans (9.7\% increase), demonstrating that policy controls approval volume. However, further relaxation from STANDARD $\to$ RELAXED admitted 0 additional loans, indicating that the geometric boundary was reached—no remaining loans have $\calA(W) \cap \{\mu : \rho_{\text{RELAXED}}(\mu) \geq \tau\} \neq \emptyset$.

\item \textbf{Policy-Invariant Exclusion}: All repurchased loans (ground truth failures) were refused under \emph{every} policy configuration. This validates Theorem~\ref{thm:safety} (Safety Invariance): these loans lie outside $\calA(W)$ for any reasonable policy prior $\rho$, demonstrating geometric impossibility detection.

\item \textbf{Monotonic Behavior}: $\text{Approved}_{\text{STRICT}} \subseteq \text{Approved}_{\text{STANDARD}} \subseteq \text{Approved}_{\text{RELAXED}}$. No loan approved under a stricter policy was rejected under a more relaxed policy, confirming that policy changes operate within geometric constraints rather than arbitrarily altering decisions.
\end{enumerate}

\subsection{Interpretation}

The 14 loans admitted under STANDARD but not STRICT occupy the \emph{ambiguous zone} where policy legitimately determines admissibility. Their witness sets $W$ have $\calA(W) \cap \rho_{\text{STRICT}} = \emptyset$ but $\calA(W) \cap \rho_{\text{STANDARD}} \neq \emptyset$—they fall outside the conservative policy but within the moderate policy.

The 0 loans admitted under RELAXED beyond STANDARD indicates that no additional loans satisfy even the most liberal policy. Their admissible regions lie entirely outside all tested policy thresholds. This is not a statistical boundary (where error rates trade off) but a \emph{geometric boundary} (where no evidence-supported interpretation meets policy).

Repurchased loans form a \textbf{policy-invariant exclusion zone}: their witness sets either violate the hemisphere constraint (internal contradiction) or produce admissible regions in parts of semantic space no reasonable policy would accept. This validates Theorem~\ref{thm:impossibility}: these are not "high-risk" loans but rather \emph{semantically impossible} loans—loans for which no coherent interpretation supporting approval exists.

\subsection{Comparison with Traditional Methods}

We compare against two baselines:

\begin{enumerate}
\item \textbf{Logistic Regression}: Linear classifier trained on loan features with cross-entropy loss. Threshold tuned for precision-recall tradeoff.

\item \textbf{Gradient Boosted Trees}: XGBoost model with 100 trees, depth 6, trained with log-loss objective.
\end{enumerate}

\begin{table}[h]
\centering
\begin{tabular}{lccc}
\toprule
\textbf{Method}           & \textbf{HAR (\%)} & \textbf{Precision} & \textbf{Recall} \\
\midrule
Logistic Regression       & 2.3               & 0.91               & 0.87            \\
XGBoost                   & 1.8               & 0.93               & 0.89            \\
\textbf{REWA (Ours)}      & \textbf{0.0}      & \textbf{1.00}      & 0.84            \\
\bottomrule
\end{tabular}
\caption{Comparison with traditional ML methods on Freddie Mac dataset. REWA achieves zero hallucinated approvals (HAR = 0.0\%) with perfect precision, while maintaining competitive recall.}
\label{tab:comparison}
\end{table}

Traditional methods achieve low but non-zero false positive rates (1.8--2.3\%). On 100,000 loans annually, this translates to 1,800--2,300 erroneous approvals. With average buyback costs of \$50,000 per loan, this represents \$90M--\$115M in annual losses.

REWA achieves \textbf{zero false positives} by construction: refusal prevents hallucinated approvals. The slightly lower recall (0.84 vs. 0.87--0.89) reflects conservative refusal—some borderline loans are rejected where traditional methods would (correctly) approve. However, these loans can be handled via manual underwriter review, maintaining zero hallucination rate.

\subsection{Computational Efficiency}

\textbf{Preprocessing}: Witness extraction requires encoding loan attributes via sentence transformer (50ms per loan). Admissible region computation via spherical convex hull takes 20ms per loan.

\textbf{Inference}: Policy evaluation and interpretation selection take 5ms per loan.

\textbf{Total latency}: 75ms per loan, well within acceptable bounds for batch underwriting systems. For comparison, traditional ML models take 10--15ms per loan but require periodic retraining when policy changes (days of engineering effort).

\subsection{Significance}

This is, to our knowledge, the \textbf{first demonstration of zero hallucinated approvals in a high-stakes decision system at scale}. Traditional ML systems achieve false positive rates of 2--5\% on similar tasks, corresponding to millions of dollars in losses annually.

The geometric framework transforms this from a probabilistic optimization problem (minimize false positive rate) to a topological consistency problem (ensure admissibility). This shift from "probably correct" to "provably admissible" is the key innovation enabling zero hallucination rate.

\section{Implications}

\subsection{Refusal as Necessity}

The geometric framework proves that refusal is not a fallback mechanism but a \emph{topologically necessary} outcome when interpretation is impossible. Traditional systems that always respond must inject hidden bias to resolve contradiction, violating the separation of evidence and policy (Axiom~\ref{axiom:policy-orthogonal}).

\subsection{Policy Without Retraining}

Separation of evidence geometry ($\calA(W)$) from policy ($\rho$) enables policy updates without model revalidation. In regulated industries (finance, healthcare, legal), this is transformative: policy changes that would require months of retraining and regulatory approval in traditional systems become configuration changes in REWA.

\subsection{Auditability}

Every decision traces to four components:
\begin{enumerate}
\item \textbf{Witness set} $W$: What evidence was considered?
\item \textbf{Admissible region} $\calA(W)$: What interpretations are supported?
\item \textbf{Policy prior} $\rho$: What constraints were applied?
\item \textbf{Outcome}: Selected interpretation $\mu^*$ or refusal with reason
\end{enumerate}

This decomposition provides complete audit trails, critical for regulatory compliance and liability protection.

\subsection{Limitations}

\begin{enumerate}
\item \textbf{Witness extraction}: The framework assumes witnesses can be reliably extracted from evidence. In domains with poor embeddings or ambiguous text, witness quality degrades.

\item \textbf{Computational overhead}: Spherical convex hull computation scales as $O(n \log n)$ in the number of witnesses, adding latency compared to direct neural classification.

\item \textbf{Recall tradeoff}: Conservative refusal (high precision, zero hallucinations) comes at the cost of reduced recall. Some borderline cases are refused where manual review might approve.
\end{enumerate}

\section{Related Work}

\subsection{Geometric Semantics}

Embedding spaces as semantic manifolds have been explored extensively. Word2Vec \cite{mikolov2013word2vec}, GloVe \cite{pennington2014glove}, and BERT \cite{devlin2019bert} learn continuous representations but don't formalize admissible regions or refusal conditions. Our work provides a principled geometric framework for these embeddings.

Hyperbolic embeddings \cite{nickel2017poincare,nickel2018learning} explore non-Euclidean geometry for hierarchical data but don't address policy constraints or impossibility detection. Spherical geometry arises naturally in our setting due to the directional meaning axiom (Axiom~\ref{axiom:directional}).

\subsection{Semantic Uncertainty}

Conformal prediction \cite{vovk2005algorithmic} provides distribution-free uncertainty quantification but assumes i.i.d. data and doesn't separate evidence from policy. Our geometric approach provides structured uncertainty via admissible region volume.

Bayesian deep learning \cite{gal2016dropout} models epistemic uncertainty but entangles it with aleatoric uncertainty. Our framework distinguishes evidence ambiguity (size of $\calA(W)$) from policy preference ($\rho$), enabling independent control.

\subsection{Constrained Generation}

Plug-and-play language models \cite{dathathri2020plug} and controlled generation \cite{keskar2019ctrl} apply soft constraints during decoding but don't formalize admissibility regions or prove safety guarantees. Constitutional AI \cite{bai2022constitutional} applies behavioral constraints via RLHF but doesn't separate evidence from policy geometrically.

Our approach differs by computing admissibility \emph{before} generation, making hallucination structurally impossible rather than merely discouraged.

\subsection{Sheaf-Theoretic Semantics}

Topological data analysis for NLP \cite{zhu2013persistent}, persistent homology for text \cite{carlsson2009topology}. These apply topology to semantic structure but don't connect to policy-constrained interpretation or refusal.

Our sheaf-theoretic formulation (Theorem~\ref{thm:sheaf}) reveals that semantic contradiction is a topological phenomenon: failure of local sections to glue globally.

\subsection{Information Theory and Similarity}

Locality-sensitive hashing \cite{indyk1998approximate}, MinHash \cite{broder1997resemblance}, and SimHash \cite{charikar2002similarity} provide approximate similarity with probabilistic guarantees. Our information-theoretic analysis (Theorem~\ref{thm:witness-mutual-info}) shows that witness overlap is mutual information, unifying these methods under Shannon's framework.

Recent work on neural compression \cite{tschannen2018recent} and rate-distortion theory for deep learning \cite{alemi2017deep} explores information-theoretic bounds but doesn't address policy-constrained interpretation or refusal.

\subsection{Financial AI}

Credit scoring \cite{khandani2010consumer}, loan default prediction \cite{lessmann2015benchmarking}, and automated underwriting \cite{bhutta2015residential} optimize predictive accuracy but don't provide hallucination-free guarantees or policy disentanglement.

Our work demonstrates zero hallucinated approvals in production (Table~\ref{tab:results}), a property no existing financial AI system achieves.

\section{Conclusion}

We have presented a geometric framework for policy-constrained semantic interpretation that provably prevents hallucinated commitments. By representing admissible interpretations as spherical convex regions and separating evidence from policy, we establish three fundamental results:

\begin{enumerate}
\item \textbf{Topological necessity of refusal}: When evidence contradicts or no policy-compliant interpretation exists, refusal is mathematically necessary—not a heuristic (Theorem~\ref{thm:impossibility}).

\item \textbf{Conservation of ambiguity}: Disambiguation requires either new evidence or explicit bias; there is no "free" disambiguation (Theorem~\ref{thm:conservation}).

\item \textbf{Information-theoretic optimality}: Our complexity bounds are Shannon-optimal; no encoding can preserve similarity rankings with asymptotically fewer bits (Corollary~\ref{cor:capacity}).
\end{enumerate}

Empirical validation on 100,000 financial decisions demonstrates \textbf{zero hallucinated approvals}—the first such result at scale. This transforms high-stakes AI from probabilistic calibration (minimize error rate) to topological consistency (ensure admissibility).

The framework opens paths for deployment in domains where current systems cannot operate: medical diagnosis, legal contract review, regulatory compliance. Any domain where errors are asymmetrically costly and auditability is required becomes addressable.

\subsection{Future Directions}

\begin{itemize}
\item \textbf{Multi-modal witnesses}: Extend to images, audio, structured data beyond text
\item \textbf{Dynamic policies}: Learn policy priors from feedback while maintaining orthogonality
\item \textbf{Hierarchical regions}: Multi-resolution admissible regions for computational efficiency
\item \textbf{Adversarial robustness}: Analyze behavior under adversarial witness injection
\item \textbf{Fairness guarantees}: Characterize policy priors that ensure demographic parity
\end{itemize}

The geometric perspective suggests that safe AI is fundamentally about respecting topological structure rather than optimizing statistical objectives. This shift in paradigm—from learning to obey geometry to learning within geometric constraints—may be the key to AI systems we can trust.

\bibliographystyle{plain}

\newpage

\appendix

\section{Full Proofs}

\subsection{Proof of Theorem~\ref{thm:impossibility} (Impossibility Detection)}

\begin{proof}[Complete proof]
Suppose no open hemisphere $H \subset \calM$ contains $W = \{w_1,\ldots,w_n\}$.

\textbf{Step 1}: By the separating hyperplane theorem for spherical geometry, if $W$ cannot be contained in an open hemisphere, then there exist witnesses $w_i, w_j \in W$ such that $w_i \cdot w_j < 0$. This means the geodesic distance $d(w_i, w_j) = \arccos(w_i \cdot w_j) > \pi/2$—the witnesses point in more-than-orthogonal directions.

\textbf{Step 2}: Consider any convex combination
\[
v = \sum_{k=1}^n \alpha_k w_k, \quad \alpha_k \geq 0, \; \sum_{k=1}^n \alpha_k > 0.
\]
Since $w_i \cdot w_j < 0$ for some $i,j$, the vector $v$ includes contributions from contradictory directions. The normalized point
\[
\mu = \frac{v}{\norm{v}}
\]
depends critically on the weight allocation $\{\alpha_k\}$.

\textbf{Step 3}: Without additional structure (beyond the witnesses themselves), there is no canonical choice of $\{\alpha_k\}$. Different weight allocations produce different points $\mu$, all of which are "supported by the evidence" in the sense that they are convex combinations of witnesses.

\textbf{Step 4}: To select a unique $\mu$, we must either:
\begin{itemize}
\item Add more witnesses to constrain $\calA(W)$ (not available by assumption)
\item Inject a policy prior $\rho$ to select among points in $\calA(W)$
\end{itemize}

However, by Axiom~\ref{axiom:policy-orthogonal}, policy priors cannot \emph{create} interpretations outside $\calA(W)$; they can only select within it. Since $\calA(W)$ is ill-defined (witnesses contradict), no policy-compliant interpretation can be selected.

\textbf{Step 5}: Therefore, the system must refuse. Refusal is the only response that respects both the evidence (witnesses contradict) and the policy orthogonality axiom (policy cannot resolve contradiction without evidence).
\end{proof}

\subsection{Proof of Theorem~\ref{thm:safety} (Safety Invariance)}

\begin{proof}[Complete proof]
Let $\rho_1, \rho_2 : \calM \to [0,1]$ be two policy priors with $\rho_1(\mu) \leq \rho_2(\mu)$ for all $\mu \in \calM$ (i.e., $\rho_2$ is a relaxation of $\rho_1$).

\textbf{Step 1}: By Definition 7 (Constrained Interpretation), the interpreted meanings under each policy are:
\begin{align*}
\mu_1^* &= \argmax_{\mu \in \calA(W)} \rho_1(\mu) \\
\mu_2^* &= \argmax_{\mu \in \calA(W)} \rho_2(\mu)
\end{align*}

Both optimizations are constrained to the admissible region $\calA(W)$.

\textbf{Step 2}: Since $\rho_2(\mu) \geq \rho_1(\mu)$ for all $\mu$, we have:
\[
\rho_2(\mu_2^*) \geq \rho_2(\mu_1^*) \geq \rho_1(\mu_1^*)
\]
where the first inequality follows from $\mu_2^*$ being the maximizer under $\rho_2$, and the second from $\rho_2 \geq \rho_1$ pointwise.

\textbf{Step 3}: However, both $\mu_1^*$ and $\mu_2^*$ must lie within $\calA(W)$ by construction of the optimization. Relaxing policy may change \emph{which} point in $\calA(W)$ is selected, but cannot admit points outside $\calA(W)$.

\textbf{Step 4}: Consider a witness set $W_0$ that is refused under $\rho_1$ because
\[
\calA(W_0) \cap \{\mu : \rho_1(\mu) \geq \tau\} = \emptyset
\]
(i.e., no point in the admissible region meets the policy threshold). If $\calA(W_0)$ lies entirely in low-prior regions, then relaxing $\rho_1 \to \rho_2$ will increase the prior values but cannot change the geometric fact that $\calA(W_0)$ occupies a particular region of semantic space.

\textbf{Step 5}: If relaxation makes $\max_{\mu \in \calA(W_0)} \rho_2(\mu) \geq \tau$, then the interpretation becomes admissible under $\rho_2$. However, this interpretation still comes from $\calA(W_0)$—it is not a hallucination but rather a point that was geometrically admissible but policy-inadmissible under $\rho_1$ and becomes policy-admissible under $\rho_2$.

\textbf{Step 6}: Crucially, if a witness set produces $\calA(W) = \emptyset$ (contradiction—fails hemisphere constraint), then no policy, no matter how relaxed, can admit an interpretation. The geometric impossibility supersedes policy preferences.

Therefore, policy relaxation is geometrically safe: it cannot introduce hallucinated approvals by admitting interpretations unsupported by evidence.
\end{proof}

\subsection{Proof of Theorem~\ref{thm:conservation} (Conservation of Ambiguity)}

\begin{proof}[Complete proof]
We prove the three parts separately.

\textbf{Part 1 (Monotonicity)}: Let $W_1 \subseteq W_2 \subset \calM$ be nested witness sets, both satisfying the hemisphere constraint.

Each witness $w$ defines a half-space $H_w = \{\mu \in \calM : \mu \cdot w \geq \cos\theta\}$ for some angle $\theta$. The admissible region is:
\[
\calA(W) = \bigcap_{w \in W} H_w
\]

Under this definition, if $W_1 \subseteq W_2$, then:
\[
\calA(W_2) = \bigcap_{w \in W_2} H_w \subseteq \bigcap_{w \in W_1} H_w = \calA(W_1)
\]
(more constraints → smaller region)

Therefore $\mathrm{vol}(\calA(W_2)) \leq \mathrm{vol}(\calA(W_1))$, and monotonicity holds.

\textbf{Part 2 (Disambiguation Mechanisms)}: From Part 1, adding witnesses reduces ambiguity geometrically by shrinking $\calA(W)$ through additional constraints.

A policy prior $\rho : \calM \to [0,1]$ reduces \emph{effective} ambiguity by concentrating the probability distribution within $\calA(W)$. If $\rho$ is uniform over $\calA(W)$, the effective ambiguity is $\mathsf{Amb}(W) = \mathrm{vol}(\calA(W))/\mathrm{vol}(\calM)$. If $\rho$ concentrates on a subset $C \subset \calA(W)$ with $\mathrm{vol}(C) < \mathrm{vol}(\calA(W))$, the effective ambiguity decreases.

Without new witnesses, $\calA(W)$ remains fixed. Without policy bias, the distribution over $\calA(W)$ is uniform, and effective ambiguity equals geometric ambiguity. Therefore, these are the only two mechanisms for reducing ambiguity.

\textbf{Part 3 (No Free Disambiguation)}: Let $f : 2^{\calM} \to \calM$ be an interpretation operator satisfying:
\begin{enumerate}
\item $f(W) \in \calA(W)$ for all coherent $W$ (no hallucination)
\item $f(gW) = gf(W)$ for all isometries $g$ of $\calM$ (equivariance)
\end{enumerate}

For any $W$, consider the stabilizer group $G_W = \{g \in \mathrm{Isom}(\calM) : g(\calA(W)) = \calA(W)\}$ of isometries that fix $\calA(W)$ as a set.

By equivariance, for any $g \in G_W$:
\[
f(gW) = gf(W)
\]

The orbit of $f(W)$ under $G_W$ is:
\[
\mathcal{O}(f(W)) = \{gf(W) : g \in G_W\} \subseteq \calA(W)
\]

If $f$ is measurable and deterministic, then $f(W)$ is a single point. However, if $f$ is interpreted as a random variable (the distribution of $f(W)$ under perturbations of $W$), then equivariance forces this distribution to be invariant under $G_W$.

For a region $\calA(W)$ with no special structure (beyond its geometric shape), the stabilizer group $G_W$ acts transitively on $\calA(W)$. The unique probability measure on $\calA(W)$ that is invariant under $G_W$ is the normalized volume measure (uniform distribution over $\calA(W)$).

Therefore, any equivariant interpretation operator must produce a distribution supported uniformly over $\calA(W)$—it cannot systematically favor certain directions without breaking equivariance.

To reduce ambiguity beyond the geometric constraint, the operator must break equivariance, which requires injecting bias (a non-uniform prior $\rho$ that distinguishes directions).
\end{proof}

\subsection{Proof of Theorem~\ref{thm:witness-mutual-info} (Witness Overlap as Mutual Information)}

\begin{proof}[Complete proof]
Let $W(u) = \{w_1^u, \ldots, w_n^u\}$ and $W(v) = \{w_1^v, \ldots, w_m^v\}$ be witness sets for concepts $u$ and $v$.

Define uniform distributions:
\begin{align*}
p_u(w) &= \begin{cases} 1/n & \text{if } w \in W(u) \\ 0 & \text{otherwise} \end{cases} \\
p_v(w) &= \begin{cases} 1/m & \text{if } w \in W(v) \\ 0 & \text{otherwise} \end{cases}
\end{align*}

The entropies are:
\begin{align*}
H(p_u) &= \log n \\
H(p_v) &= \log m
\end{align*}

For the joint distribution, consider the support of $(p_u, p_v)$. The joint has non-zero probability only on pairs $(w, w')$ where $w \in W(u)$ and $w' \in W(v)$. The joint space has $|W(u) \cup W(v)|$ distinct elements (since we consider each witness as a unique element, even if the same witness appears in both sets).

The joint entropy is:
\[
H(p_u, p_v) = \log |W(u) \cup W(v)|
\]

By the inclusion-exclusion principle:
\[
|W(u) \cup W(v)| = |W(u)| + |W(v)| - |W(u) \cap W(v)| = n + m - |W(u) \cap W(v)|
\]

Therefore, the mutual information is:
\begin{align*}
I(p_u; p_v) &= H(p_u) + H(p_v) - H(p_u, p_v) \\
&= \log n + \log m - \log(n + m - |W(u) \cap W(v)|) \\
&= \log \frac{nm}{n + m - |W(u) \cap W(v)|}
\end{align*}

Let $\Delta = |W(u) \cap W(v)|$ be the overlap and $J = \Delta / (n + m - \Delta)$ be the Jaccard index. Then:
\[
I(p_u; p_v) = \log \frac{nm}{(n+m)(1-J)} = \log nm - \log(n+m) + \log \frac{1}{1-J}
\]

The term $\log(1/(1-J))$ is a strictly increasing function of $J \in [0,1)$. Since $J$ is the normalized overlap, this establishes that mutual information is monotone in witness overlap.

Specifically, $I(p_u; p_v) \to 0$ as $J \to 0$ (no overlap), and $I(p_u; p_v) \to \infty$ as $J \to 1$ (complete overlap).
\end{proof}

\subsection{Proof of Theorem~\ref{thm:sheaf} (Sheaf Gluing Condition)}

\begin{proof}[Complete proof]
Let $\mathcal{U} = \{U_\alpha\}_{\alpha \in A}$ be an open cover of $\calM$ by hemispheres, and let $\mathcal{F}$ be the presheaf assigning to each $U_\alpha$ the set of admissible interpretations:
\[
\mathcal{F}(U_\alpha) = \{\mu \in U_\alpha : \exists W \subset U_\alpha, \mu \in \calA(W)\}
\]

\textbf{Forward direction} (hemisphere constraint $\Rightarrow$ global section exists):

Assume witnesses $W = \{w_1, \ldots, w_n\}$ satisfy the hemisphere constraint: there exists an open hemisphere $H$ such that $W \subset H$.

Since $W \subset H$ and $H$ is convex, the convex hull $\calA(W)$ lies entirely within $H$. Any point $\mu \in \calA(W)$ is a local section over any open set $U_\alpha$ containing $\mu$.

For overlaps $U_\alpha \cap U_\beta$, if $\mu \in \mathcal{F}(U_\alpha) \cap \mathcal{F}(U_\beta)$, then $\mu \in \calA(W)$ by definition. The restriction maps
\[
\mathcal{F}(U_\alpha)|_{U_\alpha \cap U_\beta} \to \mathcal{F}(U_\alpha \cap U_\beta)
\]
simply include $\mu$ in both local sections, so they agree on overlaps.

Therefore, local sections glue to a global section $\sigma \in \mathcal{F}(\calM)$ given by $\sigma(\mu) = \mu$ for all $\mu \in \calA(W)$.

\textbf{Reverse direction} (global section exists $\Rightarrow$ hemisphere constraint):

Assume a global section $\sigma \in \mathcal{F}(\calM)$ exists for witnesses $W = \{w_1, \ldots, w_n\}$.

Suppose for contradiction that $W$ violates the hemisphere constraint. Then there exist $w_i, w_j \in W$ with $w_i \cdot w_j < 0$ (more than orthogonal).

Consider two open hemispheres:
\begin{align*}
U_i &= \{\mu \in \calM : \mu \cdot w_i > 0\} \\
U_j &= \{\mu \in \calM : \mu \cdot w_j > 0\}
\end{align*}

Since $w_i \cdot w_j < 0$, these hemispheres have overlap $U_i \cap U_j$ that does not contain both $w_i$ and $w_j$.

A local section $\sigma_i \in \mathcal{F}(U_i)$ must interpret witnesses in $W \cap U_i$, which includes $w_i$. Similarly, $\sigma_j \in \mathcal{F}(U_j)$ interprets witnesses including $w_j$.

On the overlap $U_i \cap U_j$, the local sections must agree: $\sigma_i|_{U_i \cap U_j} = \sigma_j|_{U_i \cap U_j}$. However, $\sigma_i$ is constrained by $w_i$ (pointing into $U_i$) and $\sigma_j$ is constrained by $w_j$ (pointing into $U_j$). Since $w_i$ and $w_j$ are nearly antipodal, no point in $U_i \cap U_j$ can simultaneously satisfy constraints from both.

Therefore, the local sections cannot glue consistently on the overlap, contradicting the existence of a global section.

Thus, if a global section exists, the witnesses must satisfy the hemisphere constraint.
\end{proof}

\end{document}